\documentclass[a4paper,10pt]{llncs}

\usepackage{import}

\usepackage{ifthen}
\usepackage{url}
\usepackage{amssymb}
\usepackage{amsmath}
\usepackage{import}
\usepackage{xparse}
\usepackage{amsmath}
\usepackage[usenames,dvipsnames]{xcolor}
\usepackage{xspace}
\usepackage{datatool}
\usepackage{glossaries}

\providecommand\m[1]{\ensuremath{#1}\xspace}
\renewcommand{\m}[1]{\ensuremath{#1}\xspace}
\newcommand{\trval}[1]{\m{\mathbf{#1}}}

	\newcommand{\limplies}{\Rightarrow}
	
	\newcommand{\lequiv}{\Leftrightarrow}

	\newcommand{\lrule}{\leftarrow}
	\newcommand{\cause}{\stackrel{c}{\lrule}}
	
	\newcommand{\ltrue}{\trval{t}}
	\newcommand{\lfalse}{\trval{f}}
	\newcommand{\lunkn}{\trval{u}}
	
	\newcommand{\true}{\m{\top}}
	\newcommand{\false}{\m{\bot}}
	\newcommand{\voc}{\m{\Sigma}}

	\newcommand{\struct}{\m{I}}
	
	\newcommand{\I}{\m{\mathcal{I}}}

	\newcommand{\theory}{\m{\mathcal{T}}}

	\newcommand{\D}{\m{\Delta}}

	\newcommand{\lit}{\m{l}}

	\NewDocumentCommand\inter{g+g}{%
	  \IfNoValueTF{#1}
	    {\struct}
	    {\m{#1^{#2}}}}

	\newcommand{\bool}{\m{\mathbb{B}}}

	\renewcommand{\int}{\m{\mathbb{Z}}}

	\newcommand{\leqp}{\m{\leq_p}}

	\newcommand{\leqt}{\m{\leq_t}}

	\NewDocumentCommand\subs{g+g}{%
	  \IfNoValueTF{#1}
	    {\m{/}}
	    {\m{#1/ #2}}}

	\newcommand{\logicname}[1]{\textsc{#1}\xspace}

	\newcommand{\foid}{\logicname{FO(ID)}}

\newcommand{\ouracronym}[3]{%
	\newacronym{#1}{#2}{#3}
	\expandafter\newcommand\csname #1\endcsname{\gls{#1}\xspace}%
}
	\ouracronym{FO}{FO}{first-order logic}
	\ouracronym{PC}{PC}{propositional calculus}
	\ouracronym{MX}{MX}{Model Expansion}
	\ouracronym{MO}{MO}{Model Optimization}
	\ouracronym{ASP}{ASP}{Answer Set Programming}
	\ouracronym{TP}{TP}{Theorem Proving}
	\ouracronym{LP}{LP}{Logic Programming}
	\ouracronym{CP}{CP}{Constraint Programming}
	\ouracronym{FP}{FP}{Functional Programming}
	\ouracronym{KR}{KR}{Knowledge Representation}
	\ouracronym{CSP}{CSP}{Constraint Satisfaction Problem}
	\ouracronym{SMT}{SMT}{SAT Modulo Theories}
	\ouracronym{KBS}{KBS}{Knowledge Base System}
	\ouracronym{NNF}{NNF}{Negation Normal Form}
	\ouracronym{FNNF}{FNNF}{Flat Negation Normal Form}
	\ouracronym{DefNNF}{DefNNF}{Definition Negation Normal Form}
	\ouracronym{DEFNF}{DEFNF}{Definition Normal Form}
	\ouracronym{CDCL}{CDCL}{Conflict-Driven Clause-Learning}
	\ouracronym{WFS}{WFS}{Well-Founded Semantics}
	\ouracronym{LCG}{LCG}{Lazy Clause Generation}
	\ouracronym{AEL}{AEL}{Autoepistemic Logic}
	\ouracronym{OEL}{OEL}{Ordered Epistemic Logic}
	\ouracronym{AFT}{AFT}{Approximation Fixpoint Theory}

	\makeatletter
	\def\ifenv#1{
	\def\@tempa{#1}%
	\def\@ttempa{#1*}%
	\ifx\@tempa\@currenvir
	\expandafter\@firstoftwo
	\else
	\expandafter\@secondoftwo
	\fi
	}
	\makeatother

	\newcommand{\ddrule}[4]{\ensuremath{#1 \leftarrow #2 & \{#3\} & #4}}
	\newcommand{\drule}[2]{\ensuremath{#1 & \leftarrow & #2}}

	\newcommand{\darule}[4]{\ensuremath{#1 \leftarrow #2 & \{#3\} & #4}}
	\newcommand{\arule}[2]{\ensuremath{#1 \, &\leftarrow \, #2}}

	\newenvironment{ldef}{\left\{\begin{array}{l@{ \,}l@{\,}l}}{\end{array}\right\}}
	\newenvironment{ltheo}{\[\begin{array}{l}}{\end{array}\]\ignorespacesafterend}

	\newcommand{\LNDRule}[2]{
	\ifenv{array}
	{\drule{#1}{#2}}
	{ \ifenv{align}
		{\arule{#1}{#2}}
		{\ifenv{align*}
		{\arule{#1}{#2}}
		{ERROR: using LDRule in unsupported environment: \@currenvir}
		}
	}
	}

	\newcommand{\LDRule}[4]{
	\ifenv{array}
	{\ddrule{#1}{#2}{#3}{#4}}
	{ \ifenv{align}
		{\darule{#1}{#2}{#3}{#4}}
		{\ifenv{align*}
		{\darule{#1}{#2}{#3}{#4}}
		{ERROR: using LDRule in unsupported environment: \@currenvir}
		}
	}
	}

	\NewDocumentCommand\LRule{m+g+g+g}{%
		\IfNoValueTF{#2}%
		{#1.&}{%
		\IfNoValueTF{#3}
		{\LNDRule{#1}{#2.}}
		{\LDRule{#1}{#2.}{#3}{#4}}%
		}
	}

	\NewDocumentCommand\CLRule{m+g}{%
	\ifenv{array}
	{\cdrule{#1}{#2}}
	{ \ifenv{align}
		{\carule{#1}{#2}}
		{\ifenv{align*}
			{\carule{#1}{#2}}
			{ERROR: using CLRule in unsupported environment: \@currenvir}
		}
	}
	}

	\NewDocumentCommand\carule{m+g}{%
		\IfNoValueTF{#2}
			{\ensuremath{#1.}}
			{\ensuremath{#1 \, &\cause \, #2}}}
	\NewDocumentCommand\cdrule{m+g}{%
		\IfNoValueTF{#2}
			{\ensuremath{#1.}}
			{\ensuremath{#1 & \cause & #2}}}

	\newcommand{\algrule}[4]{
	\hbox{{#1}:}& 
	\quad #2 ~\longrightarrow~ #3 
	\hbox{~ if } #4\\
	}

	\newcommand{\AlgoRule}[4]{
	\ifenv{array}
	{\algrule{#1}{#2}{#3}{#4}}
		{ERROR: using AlgoRule in unsupported environment: \@currenvir}
	}

\newcommand{\commentstyle}{\color{Gray}}
	\usepackage[final]{listings}

	\lstdefinelanguage{idp}{
		morekeywords=[1]{query(}, %to ignore procedure calls to query
		morekeywords=[2]{namespace,vocabulary,theory,structure,procedure,term,set,formula, spec, specification,query},
		morekeywords=[3]{include,using,type,isa,contains,partial,extern,LFD,GFD,constructed,from,constraint,pred,supertype,of,subtype,define},
		morekeywords=[4]{int,float,char,string,nat},
		morekeywords=[5]{if,then,else,for,end},
		morecomment=[s]{/*}{*/},	
		morecomment=[l]{//}
	}
	\newcommand{\ignore}[1]{}

	\newboolean{nocomments}
	\setboolean{nocomments}{false}

	\newboolean{commentmargin}
	\setboolean{commentmargin}{true}

	\newcommand{\namedcomment}[3]{%
		\ifthenelse{\boolean{nocomments}}%
		{}%IF no comments, write nothing
		{%Otherwise
			\ifthenelse{\boolean{commentmargin}}%
				{ {\color{#3} \marginpar{\color{#3}\sc #2}#1}  }%Name in margin
				{  {\color{#3} {\sc #2}: #1}  }%Name not in margin
		}%
	}
	\newcommand{\mnamedcomment}[3]{\ifthenelse{\boolean{nocomments}}{}{{\marginpar{ \color{#3}{\sc #2}:#1}}}}

	\newcommand{\todo}[1]{\namedcomment{#1}{TODO}{blue}}

	\newcommand{\bart}[1]{\namedcomment{#1}{bb}{red}}

	\newcommand{\joachim}[1]{ \namedcomment{#1}{jj}{Sepia}}
	
	\usepackage{soul}

\usepackage{etoolbox}

\newcommand\setcitation[2]{%
  \csdef{mycommoncitation#1}{#2}}
\newcommand\getcitation[1]{%
  \csuse{mycommoncitation#1}}

\setcitation{IDP}{WarrenBook/DeCatBBD14}
\setcitation{idp}{WarrenBook/DeCatBBD14}
\setcitation{fodot}{tocl/DeneckerT08}
\setcitation{CP}{Apt03}
\setcitation{cp}{Apt03}
\setcitation{EZCSP}{lpnmr/Balduccini11}
\setcitation{KR}{Baral:2003}
\setcitation{ASPComp2}{lpnmr/DeneckerVBGT09}
\setcitation{ASPComp3}{journals/tplp/CalimeriIR14}
\setcitation{ASPComp4}{conf/lpnmr/AlvianoCCDDIKKOPPRRSSSWX13}
\setcitation{CPSupport}{ictai/DeCat13}
\setcitation{CPsupport}{ictai/DeCat13}
\setcitation{functionDetection}{iclp/DeCatB13}
\setcitation{fodot2asp}{DeneckerLTV12} %TODO replace by journal publication if one is publishedr
\setcitation{Tarskian}{DeneckerLTV12} %Same as the one above 
\setcitation{Inca}{iclp/DrescherW12}
\setcitation{csp2asp}{ijcai/DrescherW11}
\setcitation{DPLLT}{cav/GanzingerHNOT04}
\setcitation{AspInPractice}{synthesis/2012Gebser}
\setcitation{clasp}{ai/GebserKS12}
\setcitation{oclingo}{kr/GebserGKOSS12}
\setcitation{gringo}{lpnmr/GebserST07}
\setcitation{cmodels}{aaai/GiunchigliaLM04}
\setcitation{inputster}{tplp/Jansen13}
\setcitation{DLV}{tocl/LeonePFEGPS06}
\setcitation{LearningPaper}{TPLP/BruynoogheBBDDJLRDV} %TODO replace by published version
\setcitation{clog}{iclp/BogaertsVDV14} %TODO replace by journal publication if one is published
\setcitation{foc}{iclp/BogaertsVDV14}
\setcitation{inferenceClog}{ecai/BogaertsVDV14}
\setcitation{examplesClog}{nmr/BogaertsVDV14b} %TODO replace by better publication if possible
\setcitation{AFT}{DeneckerBV12}
\setcitation{KBS}{iclp/DeneckerV08}
\setcitation{KBPE}{inap/DePooterWD11}
\setcitation{lazyGrounding}{jair/CatDBS15} 
\setcitation{lazygrounding}{jair/CatDBS15} 
\setcitation{ASP}{marek99stable}
\setcitation{satid}{sat/MarienWDB08}
\setcitation{lazyclausegeneration}{constraints/OhrimenkoSC09}
\setcitation{FP}{ACMCS/Hudak89}
\setcitation{GroundingWithBounds}{jair/WittocxMD10}
\setcitation{SAT}{faia/SilvaLM09}
\setcitation{LTC}{iclp/Bogaerts14}
\setcitation{SPSAT}{ictai/DevriendtBMDD12}
\setcitation{BreakID}{sat/DevriendtBBD16}
\setcitation{breakid}{sat/DevriendtBBD16}
\setcitation{LCG}{stuckeyLCG}
\setcitation{MiniZinc}{conf/cp/NethercoteSBBDT07}
\setcitation{minizinc}{conf/cp/NethercoteSBBDT07}
\setcitation{amadini}{cpaior/AmadiniGM13}
\setcitation{bootstrapping}{lash/BogaertsJDJBD14}
\setcitation{GroundedFixpoints}{ai/BogaertsVD15}
\setcitation{LogicBlox}{datalog/GreenAK12}
\setcitation{proB}{journals/sttt/LeuschelB08}
\setcitation{NaturalInductions}{KR/DeneckerV14} %TODO replace by journal publication if one is published
\setcitation{LP}{jacm/EmdenK76}
\setcitation{SMT}{faia/BarrettSST09}
\setcitation{AF}{ai/Dung95}
\setcitation{ADF}{kr/BrewkaW10}
\setcitation{af}{ai/Dung95}
\setcitation{adf}{kr/BrewkaW10}
\setcitation{ADFRevisited}{ijcai/BrewkaSEWW13}
\setcitation{adfrevisited}{ijcai/BrewkaSEWW13}
\setcitation{DefaultLogic}{ai/Reiter80}
\setcitation{DL}{ai/Reiter80}
\setcitation{AEL}{mo85}
\setcitation{minisat}{sat/EenS03}
\setcitation{completion}{adbt/Clark78}
\setcitation{wasp}{lpnmr/AlvianoDFLR13}
\setcitation{minisatid}{ictai/DeCat13}
\setcitation{lcg}{stuckeyLCG}
\setcitation{CEGAR}{jacm/ClarkeGJLV03}
\setcitation{cegar}{jacm/ClarkeGJLV03}
\setcitation{CuttingPlane}{or/DantzigFJ54}
\setcitation{kodkod}{tacas/TorlakJ07}
\setcitation{cdcl}{Marques-SilvaS99}
\setcitation{CDCL}{Marques-SilvaS99}
\setcitation{relevance}{ijcai/JansenBDJD16}
\setcitation{WFS}{GelderRS91}
\setcitation{wfs}{GelderRS91}
\setcitation{UnfoundedSet}{GelderRS91}
\setcitation{stablesemantics}{iclp/GelfondL88}
\setcitation{StableSemantics}{iclp/GelfondL88}
\setcitation{shatter}{Shatter}
\setcitation{sbass}{drtiwa11a}
\setcitation{lparsemanual}{url:lparse_manual}
\setcitation{}{}
\setcitation{}{}
\setcitation{}{}
\setcitation{}{}
\setcitation{}{}
\setcitation{}{}
\setcitation{}{}
\setcitation{}{}
\setcitation{}{}
\setcitation{}{}
\setcitation{}{}
\setcitation{}{}
\setcitation{}{}
\setcitation{}{}
\setcitation{}{}
\setcitation{}{}
\setcitation{}{}
\setcitation{}{}
\setcitation{}{}
\setcitation{}{}
\setcitation{}{}
\setcitation{}{}
\setcitation{}{}
\setcitation{}{}
\setcitation{}{}
\setcitation{}{}
\setcitation{}{}
\setcitation{}{}
\setcitation{}{}
\setcitation{}{}
  
\newcommand\mycite[1]{%
      \ifcsname mycommoncitation#1\endcsname%
   \cite{\getcitation{#1}}%
  \else%
    \cite{#1}%
  \fi%
}	
  
\newcommand\mycitet[1]{%
      \ifcsname mycommoncitation#1\endcsname%
   \citet{\getcitation{#1}}%
  \else%
    \citet{#1}
  \fi%
}

\ignore{

}

\usepackage{paralist} % compactitem environment
\usepackage[english]{babel}
\usepackage[shortlabels, inline]{enumitem}
\setdefaultenum{\bfseries(1)}{}{}{}
\usepackage[square,sort]{natbib}
\usepackage{tikz}
\usetikzlibrary{arrows, shapes, positioning, backgrounds, fit, decorations.pathreplacing, matrix}

\tikzstyle{litnode}=[draw,circle,minimum size=.75cm]
\tikzstyle{justified}=[dashed]

\begin{document}

\newcommand{\atomof}[1]{\m{|#1|}}
\newcommand\pstruct{\m{\mathcal{I}}}
\newcommand\defs[1]{\m{\mathit{defs}(#1)}}
\newcommand\opens[1]{\m{\mathit{opens}(#1)}}
\newcommand{\negate}{\mathord{\sim}}
\newcommand{\Val}[2]{\m{V_{#1}(#2)}}
\newcommand{\rels}[2]{\m{\mathcal{R}_{{#1},{#2}}}}
\newcommand{\drels}[2]{\m{\mathcal{DR}_{#1}(#2)}}
\newcommand\expl{\m{\mathit{expl}}}
\newcommand{\dep}{\m{\mathit{dd}}}
\newcommand{\litsof}[1]{\m{\overline{#1}}}
\newcommand{\ptlit}{\m{\mathit{p_\theory}}}
\renewcommand\false{\color{red}{DID YOU MEAN lfalse instead of false?}}
\renewcommand\true{\color{red}{DID YOU MEAN ltrue instead of true?}}
\newcommand{\solverstate}{\m{\mathcal{S}}}
\newcommand{\theoremref}[1]{Theorem~\ref{#1}}
\newcommand{\exampleref}[1]{Example~\ref{#1}}
\newcommand{\sectionref}[1]{Section~\ref{#1}}
\newcommand{\definitionref}[1]{Definition~\ref{#1}}
\newcommand{\justdef}{\m{\D_{j}}}

\newcommand{\hyph}{-\penalty0\hskip0pt\relax} % replace with the already existing hyphens in words that SHOULD be hyphenated (to prevent hbox overfulls) to force hyphens that allow linebreaks

\renewcommand{\cite}[1]{\citep{#1}}

\pdfinfo{
/Title (Implementing a Relevance Tracker Module)
/Author (Joachim Jansen, Jo Devriendt, Bart Bogaerts, Gerda Janssens, Marc Denecker)
}

\title{Implementing a Relevance Tracker Module}
\author{Joachim Jansen$^1$, Jo Devriendt$^1$, Bart Bogaerts$^{2,1}$,\\ Gerda Janssens$^1$, Marc Denecker$^1$\\
$^1$\email{first.lastname@kuleuven.be}, $^2$\email{bart.bogaerts@aalto.fi}}
\institute{$^1$KU Leuven, Leuven, Belgium\\$^2$Aalto University, Espoo, Finland}

\maketitle

\begin{abstract}
PC(ID) extends propositional logic with inductive definitions: rule sets under the well-founded semantics. 
Recently, a notion of \emph{relevance} was introduced for this language. 
This notion determines the set of undecided literals that can still influence the satisfiability of a PC(ID) formula in a given partial assignment. 
The idea is that the PC(ID) solver can make decisions only on relevant literals without losing soundness and thus safely ignore irrelevant literals. 

One important insight that the relevance of a literal is completely determined by the current solver state.
During search, the solver state changes have an effect on the relevance of literals.
In this paper, we discuss an incremental, lightweight implementation of a relevance tracker module that can be added to and interact with an out-of-the-box SAT(ID) solver.

% Given the close relation between PC(ID) and ASP, this paper can serve as a guideline to 
\end{abstract}

\section{Introduction}

% 1: Sat old + very powerful (low-level language)
Since the addition of conflict-driven clause learning \cite{Marques-SilvaS99},
SAT solvers have made huge leaps forward. 
Now that these highly\hyph performant SAT-solvers exist, research often stretches \emph{beyond SAT} by extending the language supported by SAT with richer language constructs.
Research fields such as \SMT \mycite{SMT}, \CP \mycite{CP} in the form of lazy clause generation \mycite{LCG}, or \ASP \mycite{ASP} could be seen as following this approach.
In this paper, we focus on the logic PC(ID): the Propositional Calculus extended with Inductive Definitions~\cite{lpar/MarienWD07}.
The satisfiability problem for PC(ID) encodings is called SAT(ID)~\cite{sat/MarienWDB08}.
SAT(ID) can be formalised as SAT modulo a theory of inductive definitions and is closely related to answer set solving.
In fact, all the work we introduce in this paper is also applicable to so-called generate-define-test answer set programs.

% \bart{huh? Zo niet he! Dit zijn leugens! (lees: dit is ee ncopy paste van de intro va nde andere paper. We do not do what we promise below}
Recently we \mycite{relevance} defined a notion of \emph{relevance} for SAT(ID). 
The idea is that we identify a set of literals that can contribute to a solution, and then limit the solver to only make choices on those literals. 
We call such literals \emph{relevant}. 
Furthermore, instead of searching for a variable assignment that satisfies the PC(ID) theory, we search for a \emph{partial} assignment that contains sufficient information to guarantee satisfiability.  %opt to check whether it is {\it justified} that the given encoding is true.
Our approach is based on the notion of \emph{justifications}~\cite{DeneckerS93,lpnmr/DeneckerBS15}. %, which is closely linked to the Well-Founded Semantics~\cite{GelderRS91}.
As a small example, consider the following theory.

 \begin{ltheo}
 \ptlit.\\
\begin{ldef}
\LRule{\ptlit}{a \land b }\\
\LRule{a}{d\lor \lnot e \lor f}\\
\LRule{b}{c \lor \lnot g\lor h}\\
\LRule{e}{f\lor \lnot h\lor i}
\end{ldef}
\end{ltheo}

This theory contains one constraint, that \ptlit must hold, and a definition (between `\{' and `\}') of \ptlit in terms of variables $a$ to $i$. 
One way to check satisfiability would be to generate an assignment of all variables that satisfies the above theory (this is the classical approach to solving such problems). %show that the above theory is satisfied by this assignment.
What we do, on the other hand, is to search for a \emph{partial} assignment to these variables such that $\ptlit$ is \emph{justified} in that partial assignment. Consider for example the partial assignment where \ptlit, $a$, $b$, $c$ and $d$ are true and everything else is unknown.
In this assignment, $a$ and $b$ are \emph{justified} because $d$ and $c$ hold respectively; \ptlit is \emph{justified} because both $a$ and $b$ are justified. This suffices to determine satisfiability of the theory, without considering the definition of $e$ for instance.  % One can see that values assigned to other literals do not this justification.\footnote{Formal definitions follow later.} T
% One can see that $a$, as well as $b$ will be true because already one of its disjuncts is true, regardless of the assignment made to the variables $e\ldots i$.
% Following this, \ptlit will also be true in for any of these ``more complete'' assignments.

% Following the terminology from \citet{\refto{relevance}}, in the above example 
% Earlier work has introduced the notion of relevance~\mycite{relevance}.
Intuitively, a literal is \emph{relevant} if it can contribute to justifying the theory. In the above example, as soon as $d$ is assigned {\em true}, the variable $e$ becomes \emph{irrelevant}. From that point onwards, search should not take $e$'s defining rule into account. 

\mycitet{relevance} have studied relevance extensively from the theoretical perspective and ran experiments to determine how relevance influences the number of choices and the number of conflicts a solver encounters. 
These experiments concluded that VSIDS chooses a significant number of relevant literals, and that prevention of decisions on irrelevant literals can lead to significant performance gains in some hand-constructed examples.
The actual implementation has not yet been described in detail.
That is exactly the goal of this paper. 
The insights presented in this paper should make it possible to implement relevance in other solvers as well, for instance in ASP solvers. 
In fact, our implementation heavily relies on components that are already present in modern ASP solvers. 
To be precise, we reuse propagation mechanisms for inductive definitions and an unfounded set detection algorithm. 
% We present some preliminary experimental evaluation of the proposed techniques.  

The main contributions of this paper are 
\begin{inparaenum} 
  \item the introduction of a method to keep track of justification status in a SAT(ID) solver,
  \item the introduction of a method to keep track of relevance status in a SAT(ID) solver, and
  \item a discussion on the properties of these implementations.
\end{inparaenum}

% The rest of this paper is structured as follows. 
% In Section~\ref{sec:prelims} we present some necessary preliminaries. %provide a quick recap of the necessary preliminary concepts.
% In Section~\ref{sec:metadef}, we show how relevance can be seen as a formal concept defined on top of a SAT(ID) solver state and list the necessary interface of such a relevance tracker.
% In Section~\ref{sec:impl}, we show how this relevance tracker can be implemented.
% We conclude in Section~\ref{sec:conclusion}.

% \bart{als ik zo eens diagonaal door de paper ga, dan lijkt alsof we nergens de gegeven implemntatie evalueren. Dit hoeven geen experimenten te zijn, maar een stevige discussie van wat er goed en wat er slecht aan is, lijkt toch op zijn plaats. 
% Wat zijn de alternatieven?}
\section{Preliminiaries: SAT(ID), Relevance}
\label{sec:prelims}
Here we give a short introduction on PC(ID), SAT(ID), justifications, relevance, and recall how relevance can be exploited to improve PC(ID) solvers.
For a more elaborate exposition we refer to~\mycitet{relevance}.

\subsection{PC(ID)}
\label{sec:pcid}
We briefly recall the syntax and semantics of Propositional Calculus extended with Inductive Definitions (PC(ID))~\cite{phd/Marien09}. 

A truth value is one of $\{\ltrue,\lfalse,\lunkn\}$; $\ltrue$ represents \emph{true}, \lfalse \emph{false} and \lunkn \emph{unknown}. The truth order $\leqt$ on truth values is given by $\lfalse\leqt\lunkn\leqt\ltrue$, the precision order $\leqp$ is given by $\lunkn \leqp \lfalse$ and $\lunkn\leqp\ltrue$. 
Let \voc be a finite set of symbols called \emph{atoms}. A \emph{literal} \lit is an atom $p$ or its negation $\lnot p$.
In the former case, we call \lit  \emph{positive}, in the latter, we call \lit \emph{negative}.
We use $\litsof{\voc}$ to denote the set of all literals over $\voc$.
If $\lit$ is a literal, we use $\atomof\lit$ to denote the atom of $\lit$, i.e., to denote $p$ whenever $\lit=p$ or $\lit = \lnot p$.
We use $\negate \lit$ to denote the literal that is the negation of $\lit$, i.e., $\negate p = \lnot p$ and $\negate \lnot p = p$.
Propositional formulas are defined as usual. We use $\varphi \limplies \psi$ for material implication, i.e., as a shorthand for $\lnot \varphi \lor \psi$. 

A \emph{partial interpretation} \pstruct is a mapping from $\voc$ to truth values. 
We use the notation $\{p_1^\ltrue,\dots,p_n^\ltrue, q_1^\lfalse,\dots,q_m^\lfalse\}$ for the partial interpretation that maps the $p_i$ to \ltrue, the $q_i$ to \lfalse and all other atoms to \lunkn.
We call a partial interpretation \emph{two-valued} if it does not map any atom to \lunkn. 
If $\pstruct$ and $\pstruct'$ are partial interpretations, we say that $\pstruct$ is less precise than $\pstruct'$ (notation $\pstruct\leqp\pstruct'$) if for all $p\in \voc$, $\pstruct(p)\leqp\pstruct'(p)$. 
If $\varphi$ is a propositional formula, we use $\varphi^\pstruct$ to denote the truth value ($\ltrue$, $\lfalse$ or $\lunkn$) of $\varphi$ in $\pstruct$, based on the Kleene truth tables \cite{Kleene38}.
If \pstruct is a partial interpretation and \lit a literal, we use $\pstruct[\lit:\ltrue]$ to denote the partial interpretation equal to \pstruct, except that it interprets $\lit$ as $\ltrue$ (and similar for \lfalse, \lunkn).
If $\voc'\subseteq \voc$,   we use the notation $\pstruct|_{\voc'}$ to indicate the {\it restriction} of $\pstruct$ to symbols in $\voc'$. This restriction is a partial interpretation of $\voc$ and satisfies $\pstruct|_{\voc'}(p) = \lunkn$ if $p \notin \voc'$ and $\pstruct|_{\voc'}(p) = \pstruct(p)$ otherwise.

In the rest of this text, when we just say \emph{interpretation}, we mean a two-valued partial interpretation. 
An interpretation \struct is completely characterised by the set of atoms $p$ such that $\struct(p)=\ltrue$. As such, slightly abusing notation, we often identify an interpretation with a set of atoms. 
% 
% A two-valued \emph{interpretation} \struct is a subset of $\voc$.
% We identify an interpretation \struct with the two-valued partial interpretation that maps $p\in I$ to \ltrue and $p\in \voc \setminus I$ to \lfalse.

An inductive definition \D over $\voc$ is a finite set of rules of the form $p\lrule \varphi$ where $p\in \voc$ and $\varphi$ is a propositional formula over $\voc$. 
We call $p$ the head of the rule and $\varphi$ the body of the rule. 
We call $p$ \emph{defined in} \D if $p$ occurs as the head of a rule in $\D$. The set of all symbols defined in \D is denoted by $\defs\D$. 
All other symbols are called \emph{open in} \D. The set of open symbols in \D is denoted $\opens\D$.  We say that a literal $\lit$ is \emph{defined in} \D if $\atomof{\lit}\in\defs\D$. 
We use the \emph{parametrised well-founded semantics} for inductive definitions \cite{lpnmr/DeneckerV07}.
That is, interpretation \struct is a model of \D  (denoted $\struct\models\D$) if $\struct$ is the well-founded model of $\D$ in context $\struct|_{\opens\D}$.
We call an inductive definition \emph{total} if for every interpretation \struct of the open symbols, the well-founded model in context \struct is a two-valued interpretation.

A \emph{PC(ID)} theory \theory over \voc is a set of propositional formulas, called constraints, and inductive definitions over \voc. Interpretation \struct is a model of \theory if \struct is a model of all definitions and constraints in \theory. 
Without loss of generality~\cite{phd/Marien09}, we assume that every PC(ID) theory is in {\bf \DEFNF}, where $\theory = \{\ptlit, \D\}$ and
% \footnote{Our implementation does not impose this syntactic restriction. In section {\color{red} TODO}, we discuss how our algorithms are adapted in case theories are not in this normal form or contain richer language constructs such as aggregates.}
\begin{compactitem}
	\item $\ptlit$ is an atom,
	\item $\D$ is an inductive definition defining \ptlit,
	\item every rule in $\D$ is of the form $p \lrule l_1\odot \dots\odot l_n$, where $\odot$ is either $\land$ or $\lor$, $p$ is an atom, and each of the $l_i$ are literals, 
\item every atom $p$ is defined in at most one rule of $\D$. 
\end{compactitem}
A rule in which $\odot$ is $\land$, respectively $\lor$ is called a \emph{conjunctive}, respectively \emph{disjunctive}, rule.
The rules in a definition \D impose a \emph{direct dependency relation}, denoted $\dep_\D$, between literals, defined as follows.
For every rule $p \lrule l_1\odot \dots\odot l_n$ in \D, $\dep_\D$ contains $(p, l_i)$ and $(\negate p, \negate l_i)$ for all $1 \leq i \leq n$.
The \emph{dependency graph} of \D is the graph $G_\D = (\litsof{\voc},\dep_\D)$.
% We use $\theory,\pstruct\models \lit$ to denote that $\lit$ holds in all models of \theory more precise than \pstruct.

For the remainder of the paper, we assume that some PC(ID) theory $\theory=\{\ptlit,\D\}$ is fixed. %; hence, we will often omit \D and/or \theory from the notations. 

It has been argued many times before \cite{Denecker98,tocl/DeneckerT08,KR/DeneckerV14} that all sensible definitions in mathematical texts are total definitions. Following these arguments, in the rest of this paper we assume \D to be a total definition. 
% In the rest of this paper, we assume \D to be a {\it total} definition \todo{argue dat dit is sensible + ref}.

The satisfiability problem for PC(ID), i.e., deciding whether a PC(ID) theory has a model, is called \emph{SAT(ID)}.
This problem is NP-complete~\cite{sat/MarienWDB08}. 

\subsection{Justifications}
Consider a directed graph $G= (V,E)$, with $V$ a set of nodes and $E$ a set of edges.
If $G$ contains an edge from $l$ to $l'$ (i.e., $(l,l') \in E$), we say that $l$ is a parent of $l'$ in $G$ and that $l'$ is a child of $l$ in $G$.
A node $l$ is called a \emph{leaf} of $G$ if it has no children in $G$; otherwise it is called \emph{internal} in $G$.
Let $G' = (V',E')$ be another graph.
We define the union of two graphs (denoted $G\cup G'$) as the graph with vertices $V \cup V'$ and edges $E\cup E'$. 

Suppose $\lit$ is a literal with $p=\atomof\lit$ and $p\in\defs\D$ with defining rule $p \lrule l_1\odot \dots\odot l_n$.
A set of literals $J_d$ is a \emph{direct justification} of $\lit$ in \D if one of the following holds:
\begin{compactitem}
	\item $\lit=p$, $\odot$ is $\land$, and $J_d=\{l_1,\dots,l_n\}$,
	\item $\lit=p$, $\odot$ is $\lor$, and  $J_d=\{l_i\}$ for some $i$,
	\item $\lit=\lnot p$, $\odot$ is $\land$, and $J_d=\{\negate l_i\}$ for some $i$,
	\item $\lit=\lnot p$, $\odot$ is $\lor$, and  $J_d=\{\negate l_1,\dots,\negate l_n\}$.
\end{compactitem}
Note that a direct justification of a literal can only contain children of that literal in the dependency graph.

A \emph{justification}~\cite{DeneckerS93,lpnmr/DeneckerBS15} $J$ of a definition \D is a subgraph of $G_\D$, such that each internal node $\lit\in J$ is a defined literal and the set of its children is a direct justification of \lit in \D.
We say that $J$ \emph{contains} \lit if $\lit$ occurs as node in $J$. 
A justification is \emph{total} if none of its leaves are defined literals.
A justification can contain \emph{cycles}.\footnote{In this text, we assume that \D is finite; in this case cycles are simply loops in the graph. The infinite case is a bit more subtle, and an adapted definition of cycle is required to maintain all results presented below. }
A cycle is called \emph{positive} (resp. \emph{negative}) if it contains only positive (resp.~negative) literals.
It is called a \emph{mixed} cycle otherwise. 
% We say that $J$ \emph{contains} \lit if $\lit$ occurs as node in $J$. 

If $J$ is a justification and \pstruct a (partial) interpretation, we define the value of $J$ in $\pstruct$, denoted $\Val{\pstruct}{J}
$ as follows:
\begin{compactitem}
\item $\Val{\pstruct}{J} = \lfalse$ if $J$ contains a leaf $\lit$ with $\lit^\pstruct = \lfalse$ or a positive cycle (or both).
\item $\Val{\pstruct}{J} = \lunkn$ if $\Val{\pstruct}{J} \neq \lfalse$  and $J$ contains  a leaf $\lit$ with $\lit^\pstruct = \lunkn$ or a
mixed cycle (or both).
\item $\Val{\pstruct}{J} = \ltrue$  otherwise (all leaves are $\ltrue$ and cycles, if any, are negative).
\end{compactitem}

A literal $\lit$ is \emph{justified} (in \pstruct, for \theory) if there exists a total justification $J$ (of \D) that contains $\lit$ such that $\Val{\pstruct}{J}=\ltrue$.
In this case, we say that such a $J$ \emph{justifies} $\lit$ (in \pstruct, for \theory).
% We say that $J$ \emph{minimally justifies} \lit if $J$ justifies \lit and there exists no subgraph $J'$ of $J$ that also justifies \lit.
It follows from the definition that it is not possible that both $l$ and $\neg l$ are justified in the same in the same interpretation. %Note that due to the nature of the definition of when a literal is justified, it is not possible for a literal $l$ as well as its negation $\neg l$ to be justified in the same structure, for the same theory.

The {\em justification status} of an atom $p$ (in \pstruct, for \theory) is defined as follows.
The justification status of $p$ is {\em true} if and only if the literal $p$ is justified in \pstruct for \theory.
The justification status of $p$ is {\em false} if and only if the literal $\neg p$ is justified in \pstruct for \theory.
Otherwise the justification status of $p$ is {\em unknown}.

\citet{DeneckerS93} and later also \citet{lpnmr/DeneckerBS15} showed that many semantics of logic programs can be captured by justifications. 
% We recall their major result on the well-founded semantics.

% \begin{theorem}[\citet{DeneckerS93}]\label{thm:Denecker}
% Let $J$ denote any justification of definition \D.
% \begin{compactitem}
% 	\item Suppose \pstruct and $\pstruct'$ are partial interpretations. 
% 	If $\pstruct\leqp\pstruct'$ then $\Val{\pstruct}{J}\leqp\Val{\pstruct'}{J}$.
% 	\item Suppose \pstruct is an \opens{\D}-interpretation and $\pstruct'$ is the well-founded model of \D in context \pstruct. For each defined literal $\lit $, it holds that \[\lit^{\pstruct'} = \max_{\leqt}\{\Val{\pstruct}{J}\mid J\text{ a total justification of \D containing \lit}\}\]
% \end{compactitem}
% \end{theorem}

\subsection{Relevance}
Now, we recall the central definitions and theorems related to relevance. For a more detailed exposition of the theoretical foundations, we refer to \mycitet{relevance}. 

% A first observation is that instead of searching for a two-valued interpretation that is a model of the given theory \theory, one can search for a partial interpretation that justifies \ptlit instead.
\begin{theorem}[\mycite{relevance}; Theorem 3.1]\label{thm:relmain}
% 	Let $\theory= \{\ptlit,\D\}$ be a PC(ID) theory. 
	$\theory$ is satisfiable if and only if there exists a partial interpretation $\pstruct$ and a justification $J$ that justifies $\ptlit$ in $\pstruct$.
\end{theorem}

% \noindent
% Next, the definition for the set of relevant literals is given.

\begin{definition}[\mycite{relevance}; Definition 3.2] \label{def:relevant}
Given a PC(ID) theory $\theory=\{\ptlit,\D\}$ and a partial interpretation \pstruct, we inductively define the set of relevant literals, denoted $\rels{\theory}{\pstruct}$, as follows
\begin{compactitem}
 \item $\ptlit$ is relevant if $\ptlit$ is not justified,
 \item $l$ is relevant if $\lit$ is not justified and there exists some $\lit'$ such that $(\lit',\lit)\in \dep_\D$ and $\lit'$ is relevant.
\end{compactitem}
\end{definition}

The central theorem regarding relevance shows that any search algorithm that arrives in a state in which \ptlit is justified by deciding on a literal $l$ that is irrelevant can also arrive in such a state without deciding on $l$.
Hence, if a literal $l$ is irrelevant, choosing on $l$ does not help justifying \ptlit.

\begin{theorem}[\mycite{relevance}; Theorem 3.5] \label{thm:choice}
Let $\theory=\{\ptlit,\D\}$ be a PC(ID) theory. Suppose \pstruct is a partial interpretation and $\lit$ a literal such that $\pstruct(\atomof{\lit}) = \lunkn$ and $\lit$ is not relevant in \pstruct. 
If $\ptlit$ is justified in some partial interpretation  $\pstruct'$ more precise than $\pstruct$, then $\ptlit$ is also justified in $\pstruct'[\lit:\lfalse]$ and in $\pstruct'[\lit:\ltrue]$. 
\end{theorem}

Consequently, the work suggests adapting SAT(ID) solvers such that they 
\begin{inparaenum}
 \item do not make decisions on irrelevant literals, and
 \item stop searching when there are no relevant literals left.
\end{inparaenum}
This requires the underlying solver to keep track of which literals are relevant.
This task is incremental in nature: small changes to the state of the solver will result in small changes to the relevance of literals. 
Since modern solvers work with hunderds of thousands of variables and go through millions of assignment in quick succession, it is essential to do these changes as efficiently as possible.

\section{The Basic Framework}

The aim of this work is to discuss the implementation of an algorithm to keep track of relevant literals.

As said in Theorem~\ref{thm:relmain}, the solver aims to arrive at a state $\solverstate$ where $\ptlit$ is justified in $\pstruct$.
The solver does this by making decisions, performing propagation, and backtracking.
To prevent the solver from making ``useless'' decisions, we need to know whether literals are relevant or not in $\pstruct$.

We consider the underlying solver to have an internal state \solverstate of the form $\solverstate = \langle \litsof{\voc}, \theory, \pstruct\rangle$, with
\begin{inparaenum}
  \item $\litsof{\voc}$ denoting the set of literals used in the solver, 
  \item $\theory = \{\ptlit, \D\}$ a \DEFNF theory, and
  \item $\pstruct$ the current partial interpretation in the solver.
\end{inparaenum}

During the search process, the (CDCL) solver adds learned conflict clauses to the theory.
However, learned conflict clauses are logical consequences of the theory and because of this we do not consider them to be a part of the theory \theory in \solverstate.
Instead, \theory is reserved for non-learned clauses.
We assume \theory to remain static during the search process.
This assumption is valid in most {\em ground-and-solve} systems.
Recent work focuses on interleaving this process~\mycite{lazygrounding}.
Extending this work to allow for a changing theory is future work, but should be of limited complexity given the framework we present here. 
% \joachim{deze paragraaf kan maar komen nadat we \solverstate hebben gedefinieerd}

The relevance tracker needs to take into account changes in the solver state, in particular in \pstruct. Before we define the interface between the relevance tracker and the solver, we discuss the solver state and its changes.

% \bart{techinically, if you want to ``backtracking'' you need mroe in your state (order of the current assignments). But this is not very relevant for this paper of course...}

During the search process, the solver iteratively performs one of the following state changes:
\begin{itemize}
 \item $\langle \litsof{\voc}, \theory, \pstruct\rangle \mapsto \langle \litsof{\voc}, \theory, \pstruct[l:\ltrue]\rangle$ a literal $l$ becomes {\em true}, or 
 \item $\langle \litsof{\voc}, \theory, \pstruct\rangle \mapsto \langle \litsof{\voc}, \theory, \pstruct[l:\lunkn]\rangle$ a literal $l$ becomes {\em unknown}.
\end{itemize}
Note that this set of operations allows the solver to make a literal $l$ {\em false} by making literal $\negate l$ {\em true}.

In order to get the necessary information about the changes of the solver and to implement the above revision problem, the relevance tracker listens to notifications.
The relevance tracker supports the following interface to the underlying solver.
\begin{description}
%  \item [notifyNewClause] add a clause to \theory
%  \item [notifyNewRule] add a rule to \theory
 \item [notifyBecomesTrue] a literal $l$ becomes {\em true} in \pstruct
 \item [notifyBecomesUnknown] a literal $l$ becomes {\em unknown} in \pstruct
 \item [isRelevant] query whether a given literal $l$ is relevant (returns a boolean value)
\end{description}

% The {\bf notifyNewRule} method is used to pass the static \DEFNF theory \theory to the relevance tracker (and will be used to compute  justifiedness).
Methods {\bf notifyBecomesTrue} and {\bf notifyBecomesUnknown} must be cal\-led by the underlying solver when a literal has become true, respectively unknown.
The {\bf isRelevant} method is used by the solver to ask the module whether the given literal is relevant.
The relevance information allows the solver to change its underlying heuristic, selecting only relevant literals.

\ignore{
\section{Relevance as a Meta-Definition on Top of Solver State}
\label{sec:metadef}
\bart{Joachim, had jij geen refernties naar revision papers die je hier ging plaatsen. Als we (bijna) niets zeggen over revisie, dan is de waarde van ``relevance as a meta definition'' wel heel klein. Immers... De paper zou gaan over implementatie dingen...}
\joachim{Deze sectie mag er van mij part uit. Ik heb deze gedachtengang wel gevolgd bij het implementeren, maar het is niet nodig dat die in deze paper naar boven komt. Maakt de hele zaak zelfs verwarrender voor ASP mensen zelfs denk ik.}

In this section, we rephrase the (informal) definition of relevance as a set of (formal) rules. To phrase it differently, we give a definition of relevance in the logic \foid~\mycite{fodot}. \foid extends PC(ID) by supporting
\begin{inparaenum}
 \item in the vocabulary, {\em types} as a sets of values,
 \item in the vocabulary, predicate symbols with a given type signature for thei arguments, and
 \item in the theory, existential ($\exists$) or universal ($\forall$) quantification over variables of a given type.
\end{inparaenum}

We introduce a \foid meta-vocabulary to reify the internal state of the underlying solver.
The meta-vocabulary contains a single type {\bf Literal} that contains the set of all literals.
In addition, the meta-vocabulary contains the following predicate symbols (using type {\bf Literal} for all arguments).
\begin{description}
 \item [$\dep_\D(l',l)$]: literal $l$ occurs in the body of the rule defining $l'$ in \theory
 \item [justified($l$)]: literal $l$ is justified in \pstruct, for \theory
 \item [relevant($l$)]: literal $l$ is relevant in \pstruct, for \theory
\end{description}

\noindent
Using this meta-vocabulary, the definition of relevance can then be specified using a meta-definition.
\[\left\{\begin{array}{ll}
relevant(l) \leftarrow & l = \ptlit~\land~\lnot justified(\ptlit).\\
relevant(l) \leftarrow & \lnot justified(l)~\land~(\exists l' : relevant(l')~\land~\dep_\D(l',l)).
\end{array}\right\}\]

Note that the relevance of a literal $l$ and also the justifiedness of a literal is completely determined by the theory and the current status of the solver.
As said before we assume \theory and thus $\dep_\D(l',l)$ to be static.
But, whether a literal is justified or not can change when the status of the solver changes.
In the above \foid definition the only open symbol is {\bf justified}.
We can then view the computation of the changed value of $relevant/1$ based on changes in the interpretations of the open symbol {\bf justified($l$)} as a model revision task.
As such, in principle, tracking relevance can be tackled by one of the many existing model revision algorithms \cite{,,,} or database view maintenance algorithms \cite{,,,} \todo{joachim, refs?}. We, however, opted for a standalone implementation \todo{since?}, to be detailed in the following sections. 

In order to implement such a revision algorithm, the relevant tracker needs to have information about which literals are justified. 
In section~\ref{sec:justtracking} we show how also this information can be extracted from the solver, reusing existing propagation mechanisms. 
}

\section{Deriving the justification status of literals}
\label{sec:justtracking}
The definition of relevance relies on knowledge about which literals are \emph{justified} in the solver. In this section, we discuss how we implemented extraction of this information. 
We opted to implement a method that re-uses the underlying SAT(ID) solver to keep track of the justification status of literals. 
The method creates a new atom, called the ``justification atom'', for each defined atom $p$, denoted as $j(p)$.
We call a literal $j(p)$ or $\lnot j(p)$ a {\em justification literal}.

The intended interpretation of $j(p)$ is that $j(p)$ is true iff $p$ is justified, $j(p)$ is false iff $\lnot p$ is justified and $j(p)$ is unknown otherwise. 
To ensure that justification literals indeed get the right value, an extra PC(ID) definition \justdef, denoted the ``justification definition'', is added to the theory \theory. 
\justdef is constructed based on the original definition \D in the following manner.
The existing definition \D is copied, except that every defined atom $p$ is replaced with the newly created atom $j(p)$.
Thus, of all the atoms in the original definition, only the open atoms remain.
\begin{example}\label{ex:justdef}
Transforming the original definition
\[\D = \left\{\begin{array}{rrrrrrrrr}
\ptlit & \leftarrow & c_1 & \land & c_2 & \land & c_3 & \land & c_4\\
c_1    & \leftarrow & \lnot b & \lor & \lnot d\\
c_2    & \leftarrow & a & \lor & b & \lor & \lnot c\\
c_3    & \leftarrow & \lnot b & \lor & e & \lor & \lnot f\\
c_4    & \leftarrow & d & \lor & f & \lor & \lnot a\\
f & \leftarrow & b & \lor & d\\
\end{array}\right\}\]
leads to the justification definition
\[\justdef = \left\{\begin{array}{rrrrrrrrr}
j(\ptlit) & \leftarrow & j(c_1) & \land & j(c_2) & \land & j(c_3) & \land & j(c_4)\\
j(c_1)    & \leftarrow & \lnot b & \lor & \lnot d\\
j(c_2)    & \leftarrow & a & \lor & b & \lor & \lnot c\\
j(c_3)    & \leftarrow & \lnot b & \lor & e & \lor & \lnot j(f)\\
j(c_4)    & \leftarrow & d & \lor & j(f) & \lor & \lnot a\\
j(f) & \leftarrow & b & \lor & d
\end{array}\right\}\]
\end{example}

In addition to the creation of this new definition \justdef, we prohibit the solver from making choices on these justification atoms.
Because of this, the value of all $j(p)$ will be purely the result of the underlying propagation mechanism for definitions.
Our claim is now that existing propagation mechanisms will propagate \emph{exactly} those literals that are justified. 
We assume a solver that performs \emph{unit propagation} and \emph{unfounded set propagation}~\cite{lpar/MarienWD07,ai/GebserKS12}, i.e., propagation that makes all atoms in an unfounded set false.

% \begin{theorem}
% %  \bart{TODO: there should be atheorem that states correctness of our idea. Roughly of the following form:}\joachim{Ik zet me hierachter. Ik voeg ook nog een ``herschrijving'' van onze solver state toe, waarbij de nieuwe solver state ook de \justdef bevat}
%  Let $\D$ be a (total) definition and $\pstruct$ a partial interpretation in which all defined symbols of $\D$ are interpreted as \lunkn. 
%  Let $\lit$ be a defined literal in $\D$. In this case 
%  \[
%   \D, \pstruct \models \lit
%  \]
%  if and only if $\lit$ is justified in \pstruct. 
% \end{theorem}
% \begin{proof}
%  If $\lit$ is justified, by definition there must exist a justification $J$ such that
%  \begin{inparaenum}
%   \item $\Val{\pstruct}{J} = \ltrue$, 
%   \item $J$ is total, and
%   \item $J$ contains $\lit$.
%  \end{inparaenum}
%  \theoremref{thm:Denecker} can then be applied to derive $\lit^{\pstruct'} = \ltrue$ (with $\pstruct'$ the well-founded model of \D in \pstruct), because $J$ satisfies all requirements in the set-expression.
%  This also implies $\D, \pstruct \models \lit$.
% \end{proof}
% \bart{Bovenstaande stelling is niet juist. 
% (in my defence: ik had toen ik de stelling schreef er bij gezet dat ik twijfelde :p )
% 
% Tegenvoorbeeld: 
% \[q\lrule \lnot p.\qquad q\lrule p.\]
% In dit geval geldt $theory, \emptyset\models q$, maar niet dat $q$ justified is. 
% 
% Het is niet erg. Ik fix de ONDERSTAANDe stelling en we laten die hier boven gewoon weg.}

\begin{theorem}
 Let $\D$ be a (total) definition and $\pstruct$ a partial interpretation in which all defined symbols of $\D$ are interpreted as \lunkn. 
 Let $\lit$ be a defined literal in $\D$. 
 In this case $\lit$ is justified in \pstruct if and only if $\lit$ is derivable by unit propagation on the completion\footnote{The completion of a rule $p \leftarrow q$ is the underapproximation using the \FO sentence $p \lequiv q$, which demands that $p$ and $q$ hold equal truth values}\mycite{completion} of \D and unfounded set propagations.
\end{theorem}
% BART: ik heb ``rule applications'' uit de stelling weggehaald. Reden: beschouw een regel
% p <- q
% dit is de enige regel voor p en q is false, leid je af dat p false is? Indien ja... Dat is niet ECHT rule application (maar unit propagation op de completion). 
\begin{proof}[Sketch of the proof]
Intuitively, from a sequence of propagations, we can create a justification and vice versa: each justification induces a sequence of propagations. 
The correspondence is as follows. 
First for the completion, if \D contains a rule 
 $p \lrule l_1\land \dots \land l_n$, then this rule propagates $p=\ltrue$ if and only if each of the $l_i$ is true. 
 This corresponds to the justification 
 
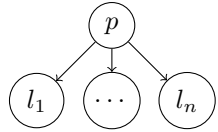
\[
  \begin{tikzpicture}

\def \n {5}
\def \radius {3cm}
\def \margin {8} % margin in angles, depends on the radius
  \node[draw, circle] at (1,0) (pnode) {$p$};

    \node[draw, circle] at (0,-1) (l1node) {$l_1$};
    \node[draw, circle] at (1,-1) (ldotsnode) {$\dots$};
    \node[draw, circle] at (2,-1) (lnnode) {$l_n$};
    \draw[->] (pnode) -- (l1node);
    \draw[->] (pnode) -- (ldotsnode);
    \draw[->] (pnode) -- (lnnode);
\end{tikzpicture}\]
And similar justification constructs can be defined for when $\lnot p$ is propagated or when the rule is disjunctive. 

Unfounded set propagation essentially corresponds to a justification of a set of negative facts by a negative cycle. 

The condition that a justification can have no mixed or positive cycles corresponds to the fact that propagation must happen in order. E.g., from the rule $p\lrule p \lor q$, $p$ can only be propagated if $q$ is true; $p$ cannot be propagated ``because $p$ is true''. 
\qed
\end{proof}

The previous theorem establishes that our approach works; a justification literal $j(p)$ will be propagated to true if $p$ is justified (note that $p$ is justified iff $j(p)$ is justified). 
% After these two theorem it is clear that our approach works.
It also explains why we use a duplicated definition: the theorem only holds if $\I$ is an $opens(\D)$ interpretation.
Since this cannot be enforced (we don't want to intrude in the solver's search), we make a copy and never make choices on the copied defined symbols.

Thus, we extend our solver state $\solverstate = \langle \litsof{\voc}, \theory, \pstruct\rangle$ to a $\solverstate' = \langle \litsof{\voc}', \theory', \pstruct', \voc' \rangle$ with
\begin{itemize}
 \item $\voc' = $ set containing the newly introduced justification atoms that the solver cannot decide on
 \item $\litsof{\voc}' = \litsof{\voc} \cup \litsof{\voc'}$
 \item $\theory' = \{\ptlit, \D'\}$ if $\theory = \{\ptlit, \D\}$ and $\D' = \D \cup \justdef$
 \item $\pstruct' = $ a partial structure over $\litsof{\voc}'$
\end{itemize}

With all this in place, we derive the interpretation for $justified(l)$ as follows.
\begin{itemize}
 \item $justified(p)$ is true if and only if $j(p)$ is true in $\pstruct'$, and
 \item $justified(\lnot p)$ is true if and only if $j(p)$ is false in $\pstruct'$.
\end{itemize}

\section{Implementing the relevance tracker}
\label{sec:impl}
The source code that implements the techniques discussed here can be found at \small \url{https://dtai.cs.kuleuven.be/static/krr/files/experiments/idp_relevance_experiments.tar.gz} \normalsize.

The solver maintains a subgraph of the dependency graph in order to keep track of the set of relevant literals.
This subgraph, denoted as the {\em relevance graph} contains all literals that are relevant and all edges between them (in the dependency graph). 
% From the definition of relevance follows that the literals in this graph correspond exactly with the set of relevant literals.
As such, the task of the tracker is to determine whether a given literal is a member of this graph or not.
We store this graph using a data structure, denoted {\tt candidate\_parents($l$)} that associates a literal with a set of literals called ``candidate parents''.
The ``candidate parents'' of \lit are the literals that are parents of \lit in the relevance graph. I.e., if \lit is irrelevant, this set is empty, otherwise, it consists of all relevant parents of \lit in the dependency graph. 
As such, it can be seen that there is an edge $(p,l)$ in the relevance graph iff $p\in~${\tt candidate\_parents($l$)}.
% If some literal $l$ is not present in the relevance graph, then {\tt candidate\_parents($l$)~=~$\emptyset$}.
% In other words, these candidate parents are the set of relevant, unjustified parents of $l$ in the dependency graph.
% This means that, using only unjustified literals, a head-child chain of literals can be constructed from $\ptlit$ to $l$, which indicates that $l$ is relevant according to \definitionref{def:relevant}. %misschien alternatief?
Thus, $l$ is relevant if and only if $l$ has a non-empty set of candidate parents.
We now describe an incremental algorithm to update the set of candidate parents for all literals if the state of the solver changes.
We prefer to keep these changes \emph{local}, i.e., to not reconstruct the entire relevance graph with each solver change. 

% When the solver has to make a decision, a heuristic (usually VSIDS~\cite{faia/2009-185}) is used to select a literal (say, $l$).
% In the adaptation the solver will query our relevance tracker to know whether $l$ is relevant using {\bf isRelevant($l$)}.
%  \bart{Blijf geen fan van de terminologie. maar heb er neit meteen een betere naam voor...}

When the solver state changes, and the set of candidate parents must be updated, care must be taken to detect and remove cyclic dependencies.
These cyclic dependencies can arise when a candidate parent is removed from a literal $l$ and the remaining candidate parents of that literal are not reachable from \ptlit anymore but still have $l$ as a candidate parent, creating a loop.
A more detailed example is given in \exampleref{ex:cycles}

\begin{example}
\label{ex:cycles}
The following definition has the cyclic dependency of $p \leftarrow q \leftarrow p$.

\[\left\{\begin{array}{rl}
\ptlit & \leftarrow  a \lor p\\
p      & \leftarrow  q\\
q      & \leftarrow  p\\
\end{array}\right\}\]
Initially \pstruct = $\emptyset$, thus nothing is justified and all literals are relevant.
Thus, $p$ has candidate parents $\{\ptlit, q\}$, and $q$ has candidate parents $\{p\}$.
\begin{figure}
\centering
 \begin{tikzpicture}
  \node []        at (1,2) (start) {};
  \node [litnode] at (1,1)   (PT)    {\ptlit};
  \node [litnode] at (0,0)   (A)     {$a$};
  \node [litnode] at (2,0)   (P)     {$p$};
  \node [litnode] at (4,0)   (Q)     {$q$}; 
  
  \draw[->]  (start) to (PT);
  \draw[->]  (PT) to  (A);
  \draw[->]  (PT) to  (P);
  \draw[->]  (P) to[out=20,in=160]  (Q);
  \draw[->]  (Q) to[out=200,in=340]  (P);
 \end{tikzpicture}
 \caption{Relevance graph for $\pstruct = \{\}$}
\end{figure}

Consider the case where $a$ becomes true and hence $\ptlit$ becomes justified.

% \bart{Shouldn't the next graph  be empty? pt is justified so nothing is relevant?}
\begin{figure}
\centering
 \begin{tikzpicture}
  \node [litnode] at (1,1) (PT) {\ptlit};
  \node [litnode] at (0,0) (A) {$a$};
  \node [litnode] at (2,0) (P) {$p$};
  \node [litnode] at (4,0) (Q) {$q$};

  \draw[->, dashed]  (P) to[out=20,in=160]  (Q);
  \draw[->, dashed]  (Q) to[out=200,in=340]  (P);
 \end{tikzpicture}
 \caption{Relevance graph for $\pstruct = \{a^{\ltrue}\}$. Remaining loop indicated with dashed edges. }

\end{figure}

Simply removing $\ptlit$ from the set of candidate parents of $p$ means that $p$ still has candidate parents $\{q\}$, which is actually just a loop leading back to $p$.
Thus, a cycle detection algorithm is needed to force $p$ and $q$ to remain loop-free. 

\end{example}

Thus, adding and removing candidate parents is a complicated matter.
In \sectionref{sec:ufsdetection} we discuss how this cycle detection is done.
For now, we use the following interface for adjusting the set of a candidate parents.
\begin{description}
 \item [notifyAddCandidateParent($l$,$l'$)] add $l'$ to the candidate parents of $l$
 \item [notifyRemoveCandidateParent($l$,$l'$)] remove $l'$ from the candidate parents of $l$
\end{description}

Our definition of candidate parents potentially changes when the following changes take place (note that we already assumed the dependency relation to be non-changeable).
\begin{itemize}
 \item A change in the justification status of $l$
 \item A change in the relevance status of a parent literal $l'$
\end{itemize}

  % \bart{Het hele event-based gedoe zou uitgelegd moeten worden. Ergens in de prelims meot staan dat voor ons een solver uit verschillende modules bestaat. Een module kan events triggeren en events observen. Een centraal mechanisme zorgt dat iedere observer op tijd en stond op de hoogte wordt gebracht als er events gebeuren waarin hij geinteresseerd is. 
  % Ook: deze events zijn standaard... In de huidige sectie dan: deze events zijn nieuw.}\joachim{ik ben hier niet helemaal mee akkoord. De bestaande solver was helemaal niet event-based en het notifyBecomesTrue event bestond nog niet. Minisatid is helemaal niet modulair opgebouwd} \bart{MinisatID is niet helemaal modulair opgebouwd. Maar dit soort notifications bestond wel (alle propagators gebruiken die). Het is toch netter om, zelfs al is het maar enkel voor de conceptuele uitleg, via events te werken dan via ``the following call must be inserted in the solver in all places where a literals value is changed''}
  % \joachim{als dat event-based stuff zo inderdaad klopt voor de gemiddelde SAT-solver ben ik akkoord. Anders vind ik dat er eerst nog een verantwoording moet komen van waarom je denk dat een implementatie gebaseerd op events een goed idee is en een goeie assumption is voor de solver. Je bent dan eigenlijk over de architectuur van de implementatie van SAT aan het spreken, iets wat mij hier (workshop paper) out of scope lijkt.}
  % \bart{Mjah, misschien is het wat verregaand...}

\noindent
Thus, we extend the interface of the relevance tracker to also support the following methods.
\begin{description}
 \item [notifyBecomesJustified($l$)] A literal $l$ goes from unjustified to justified
 \item [notifyBecomesUnjustified($l$)] A literal $l$ goes from justified to unjustified
 \item [notifyBecomesRelevant($l$)] A literal $l$ goes from irrelevant to relevant
 \item [notifyBecomesIrrelevant($l$)] A literal $l$ goes from relevant to irrelevant
\end{description}

% \bart{Er moet duidelijk gemaakt worden wat er nieuw is in deze sectie ten opzichte van de vorige paper}\joachim{nog altijd het geval? Ik weet niet hoe ik nog extra zou meoten benadrukken dat het hier gaat om de implementatie van de ideas uit de vorige paper} \bart{Het is ok. Om een of andere reden dacht ik dat het idee van ``definitie dupliceren'' al in de oude paper stond en dat dit dus niet helemaal nieuw wwas. Maar dat is niet het geval. Dus... ik verwijder de comment}
\noindent
In the following subsections we present 
\begin{inparaenum}
\item an overview of the data structures in the relevance tracker and
\item the algorithms for the methods in our interface.
\end{inparaenum}

% First, we propose a method to keep track of the justification status of literals.

% \joachim{ iets in deze aard?\\
% It is important to note that the relevance tracker cannot be implemented using a similar strategy as the one presented for the justification tracker.
% This is because the definition of relevance is a definition on top of the solver \emph{state} and not on top of atoms (in this case the opens) that already exist in the solver. 
% This is because the interpretation of the meta-vocabulary presented in \sectionref{sec:metadef} is non-monotonous with respect to the underlying interpretation.
% For example, say that the solver makes a literal $l \in \opens{\D}$ true when it was previously unknown.
% $\solverstate = \langle \litsof{\voc}, \theory, \pstruct\rangle \mapsto \solverstate' = \langle \litsof{\voc}, \theory, \pstruct[l:\ltrue]\rangle$
% This means that the literal $l$ was not justified in $\solverstate$ and was justified in ... \\
% ik weet niet wat hier precies te zeggen.

\subsection{Data Structures}
The data structures include sets and maps.
Unless specified otherwise, we use hash sets and hash maps.
The implementation uses {\tt std::unordered\_set} and {\tt std::unordered\_map} provided by the {\tt C++} standard library.

Internally, we store the dependency relation $\dep_\D$ using two maps in our module, named {\tt children} and {\tt parents}.
These data structures map a literal to a set of literals.
The first map ({\tt children}) maps a literal to its set of children in $\dep_\D$.
The second map ({\tt parents}) maps a literal to its set of parents in $\dep_\D$. %sthe set containing all head literals that depend on that body literal.
They are initialised using the {\bf notifyNewRule} method.
Once they are initialised, they remain fixed.

We use a map ({\tt to\_just\_lit}) to transform a normal literal to its justification literal ($p \mapsto j(p)$, $\lnot p \mapsto \lnot j(p)$).
For efficiency reasons, we also maintain the inverse map {\tt to\_nonjust\_lit}~$=$~{\tt to\_just\_lit$^{-1}$}.
These maps are initialised when the justification definition \justdef is created and do not change during execution afterwards.

We maintain a set of atoms ({\tt is\_just\_atom}) to identify the justification atoms that were introduced.
This set are initialised when the justification definition \justdef is created and does not change during execution afterwards.

We use round brackets to indicate the result of a map lookup, e.g., \[\text{\tt to\_just\_lit($p$) = $j(p)$}.\]

We use round brackets to do a containment check of sets.
More precisely, \[{\tt is\_just\_atom(}p{\tt )} = \text{\tt true}\] if and only if $p$ is in the set {\tt is\_just\_atom}.
As mentioned before, the underlying solver is not allowed to make decisions on literals in this set.

We maintain a map {\tt candidate\_parents} with the invariant that it maps a literal $l$ to the set of candidate parents of $l$. 
This map is dynamic throughout execution and changes to this map are performed using the {\bf notifyAddCandidateParent} and {\bf notifyRemoveCandidateParent} methods.
 
\subsection{Notification-based Algorithms}
Given that the invariant of {\tt candidate\_parents} is satisfied in solver state $\solverstate=\langle \litsof{\voc}, \theory, \pstruct\rangle$,
we wish to perform the necessary changes such that they are satisfied in solver state $\solverstate' = \langle \litsof{\voc}, \theory, \pstruct[p:tv]\rangle$ with $p$ some atom and $tv$ one of $\{\ltrue, \lfalse, \lunkn\}$.

We initiate our notification-based algorithm as follows.
If $tv = \ltrue$, then we call {\bf notifyBecomesTrue($p$)}.
If $tv = \lfalse$, then we call {\bf notifyBecomesTrue($\lnot p$)}.
If $tv = \lunkn$, then we call {\bf notifyBecomesUnknown($p$)}.

This call, in turn, can trigger other internal notifications.
The implementation of these cascading notifications ensures that {\tt candidate\_parents} will comply with its invariant in interpretation $\solverstate'$ after the designated call to  {\bf notifyBecomesTrue} or {\bf notifyBecomesUnknown} is complete.

The relevance tracker implements {\bf isRelevant($l$)} by checking whether {\tt can\-didate\_parents($l$)} maps to an empty set or not.
This is a correct representation of the relevance status of $l$ if the invariant of {\tt candidate\_parents} is satisfied.

For methods {\bf notifyBecomesTrue($l$), notifyBecomesUnknown($l$)}: the given literal can be a normal literal ($p$ or $\lnot p$) or a justification literal ($j(p)$ or $\lnot j(p)$).
The relevance tracker takes no action for normal literals.
If the given literal is a justification literal, then we retrieve the original normal literal and notify the relevance tracker that this literal has become (un)justified.
Note that we re-use the notation of $\atomof{l}$ to indicate the atom of literal $l$.
% \bart{I would rephrase this to raising an event, not calling a specific method. What if multiple modules are interested in this fact?}\joachim{ik zou dat out of scope voor deze paper beschouwen - we houden ons enkel met de relevantiemodule bezig}

\noindent
{\bf notifyBecomesTrue($l$)}: if {\tt is\_just\_atom}($\atomof{l}$), then call {\bf notifyBecomesJustified({\tt to\_nonjust\_lit($l$)})}.

\noindent
{\bf notifyBecomesUnknown($l$)}: if {\tt is\_just\_atom}($\atomof{l}$), then call {\bf notifyBecomesUnjustified({\tt to\_nonjust\_lit($l$)})}.

\noindent
{\bf notifyBecomesJustified($l$)}: call {\bf notifyRemoveAllCandidateParentsOf($l$)}.

\noindent
{\bf notifyBecomesUnjustified($l$)}: for all parents $p$ of $l$ that are relevant, call {\bf notifyAddCandidateParent($l$,$p$)}.

\noindent
{\bf notifyBecomesRelevant($l$)}: for all children $c$ of $l$, call {\bf notifyAddCandidate\-Parent($c$,$l$)}.

\noindent
{\bf notifyBecomesIrrelevant($l$)}: for all children $c$ of $l$, call {\bf notifyRemoveCan\-didateParent($c$,$l$)}.

\subsection{Maintaining watches instead of sets of candidates}
The above methods dictate how the {\tt candidate\_parents} map should be manipulated.
For efficiency reasons, the relevance tracker does not actively maintain this set of candidate parents.
Instead it keeps track of a single candidate parent as ``watched'' parent.
This watched parent is maintained using a map called {\tt watched\_parent($l$)} that maps a literal to a single parent of $l$.
The method {\bf isRelevant($l$)} now checks whether a given literal $l$ has a watched parent or not.

We only keep track of a single watched parent in order to minimize how many times a cycle detection algorithm has to be invoked.
The manipulation of the set of candidate parents, along with the invocation of a cycle detection algorithm is done as follows
\begin{description}
 \item [notifyAddCandidateParent($l$,$l'$)]
  Check for the following criteria
  \begin{itemize}
   \item $l$ does not have a watched parent yet
   \item $l$ is not justified
   \item $l'$ is relevant
   \item $l$ is a child of $l'$
  \end{itemize}
  If they are met, make {\tt watched\_parent($l$) = $l'$} and call {\bf notifyBecomes\-Relevant($l$)}.
  Note that a cyclic dependency check between $l$ and $l'$ is not needed, since $l$ could not have been a suitable watch for any other literal, as it was not relevant before.
 \item  [notifyRemoveCandidateParent($l$,$l'$)]
 If $l$ had $l'$ as its watch, remove $l'$ as watched parent of $l$.
 Try to find an alternative candidate parent $n$ such that the following hold.
  \begin{itemize}
   \item $n \neq l'$
   \item $l$ is a child of $n$
   \item $l$ is not justified
   \item $n$ is relevant
   \item Use a cycle detection algorithm to verify that setting {\tt watched\_parent($l$) = $n$} would not create a dependency cycle
  \end{itemize}
 If such $n$ can be found, set  {\tt watched\_parent($l$) = $n$}.
 If such $n$ cannot be found, call {\bf notifyBecomesIrrelevant($l$)}.
\end{description}
The implementation for the search for an alternative watch is a re-use of an existing ``unfounded set detection'' algorithm.
This algorithm is considered the fastest algorithm to achieve this task to date.
% \subsection{Invariant}
% The invariant of the algorithm is as follows: the {\tt watched\_parent($l$)} map can be converted to a tree that serves as justification for the definition of $relevant(l)$.

\subsection{Detecting Cycles}
\label{sec:ufsdetection}
For our implementation of the detection of cycles, we re-use parts of the existing {\em unfounded set propagation} algorithm~\cite{lpar/MarienWD07,ai/GebserKS12}.
This algorithm has a subcomponent that searches for cycles over negative literals.

\section{Conclusion and Future Work}
\label{sec:conclusion}

The notion of relevance is related to Magic Sets~\cite{jlp/BeeriR91,pods/BancilhonMSU86} in the field of Logic Programming in the sense that the resulting program of the magic set transformation on a program $P$ and query $Q$ does not execute parts of $P$ that do not contribute towards solving query $Q$.

One area where relevance should give great speedups is stable model counting. %, where the solver can spend a significant amount of time in satisfiable subtrees.
When \ptlit is justified, the number of solutions that this partial assignment represents is equal to $2^{n}$ with $n$ the number of unassigned open atoms in \D.
Stable model counters generally stop when the \emph{justified residual program} \cite{aaai/AzizCMS15} is empty.
Whenever this occurs, \ptlit is justified. 
However, this does not hold the other way round.
There are other cases where \ptlit is justified, but the justified residual program is non-empty. 
As such, exploiting relevance can ensure cutting out bigger parts of the search tree when model counting. 

In this paper we have presented our implementation of a relevance tracker module on top of an existing SAT(ID) solver, which consists of two methods; one to keep track of the set of justified literals and one to keep track of the set of relevant literals.
Each of them reuses existing techniques, namely the following.
\begin{itemize}
 \item Our method for keeping track of the justification status of literals. This method reuses existing propagation mechanisms. 
       We also prove the correctness of this approach if the underlying solver guarantees completeness with respect to rule application and unfounded set propagation.
 \item Our method for keeping track of the relevance status of literals reuses the existing unfounded set detection techniques~\cite{lpar/MarienWD07,ai/GebserKS12} for detection of cyclic dependencies between ``candidate parents''.
\end{itemize}

Generate-and-test \ASP programs are the most common form of \ASP programs, as can be witnessed, e.g., from the benchmarks in the latest \ASP competitions~\cite{journals/ai/CalimeriGMR16}.
Generate-and-test \ASP programs closely correspond to PC(ID) theories~\cite{DeneckerLTV12}.
This paper imposes minimal assumptions on the underlying solver, thus making it possible to translate these ideas to an \ASP context.
Since experiments have shown~\mycite{relevance} that exploiting relevance during solving can reduce the number of decisions, as well as the number of conflicts, it would be interesting to see how relevance might possibly affects existing ASP solvers. 

% Further work also includes comparing these algorithms to existing techniques in the field of view maintenance for databases.

\section*{Acknowledgements}
This research was supported by the project GOA 13/010 Research Fund
KU Leuven and projects G.0489.10, G.0357.12 and G.0922.13 of FWO
(Research Foundation - Flanders). 
Bart Bogaerts is supported by the Finnish Center of Excellence in Computational
Inference Research (COIN) funded by the Academy of Finland (grant \#251170).

\small
\bibliographystyle{plainnat}
\bibliography{idp-latex/krrlib}

\end{document}
%  \bart{ DIT IS VOOR ALS WE DIT IN DE JOURNAL VERSIE PLAATSEN:
%  Nog een bedenking. We moeten argumenteren WAAROm we de de definitie van relevance niet gewoon kunenn toevoegen aan de solver en de solver zijn ding laten doen. We weten waarom: solver state gedraagt zich niet ``monotoon''. Een literal waarvoor isjustified(l)=false geldt in I, daarvoor kan isJustified=true gelden in een meer preciese interpretatie. 
%  
%  Dus: meer precieze interpretaties geven geen meer precieze interpretatie van die literals. Dat werkt dus niet. 
%  
%  Maar, dan kunnen we zeggen: we kunnen wel IN PRINCIPE de bestaande algoritmes die op die defintiie werken kopieren. 
%  Die algoritmes kunnen al omgaan met:
%  
%  * Literal die unknown was wordt true  (search)
%  
%  * Literal die true was wordt unknown (backtracking)
%  
%  Wat ze NIET kunnen is omgaan met 
%  
%  * Literal die false was wordt true
%  
%  MAAR... Daar kunnen we eenvoudig rondwerken door twee van de bestaande updates na elkaar te doen (van false naar u en van u naar t). Dit is wat we eigenlijk doen (juist?)
%  }
%  \joachim{ik volg je laatste zin niet helemaal... we hebben nu een dedicated algo dat niet spreekt in termen van literal value assignments}
%  \bart{Wat ik eigenlijk zeg is: als we met de assignment stack zouden kunnen foefelen, dan zouden we gewoon bestaande UFS propagator kunnen gebruiken,...}